%% file: lazypsrl_arxiv.tex
\documentclass[11pt]{article} % For LaTeX2e
\usepackage{hyperref}
\usepackage{url}

\usepackage{setspace}
\usepackage{fullpage}
\usepackage{mathtools}

\usepackage{enumerate}
\usepackage{amsthm}

\usepackage[disable]{todonotes}

\usepackage[round,comma]{natbib}
\usepackage[capitalize]{cleveref}

\input{Commands}

\title{Bayesian Optimal Control of Smoothly Parameterized Systems:\\
The Lazy Posterior Sampling Algorithm}

\author{
Yasin Abbasi-Yadkori \\
Queensland University of Technology\\ 
\texttt{yasin.abbasiyadkori@qut.edu.au} \\
\and
Csaba Szepesvari \\
University of Alberta \\
\texttt{szepesva@ualberta.ca} \\
}

\begin{document}

\maketitle

\begin{abstract}
We study Bayesian optimal control of a general class of smoothly parameterized Markov decision problems. Since computing the optimal control is computationally expensive, we design an algorithm that trades off performance for computational efficiency. The algorithm is a \textit{lazy} posterior sampling method that maintains a distribution over the unknown parameter. The algorithm changes its policy only when the variance of the distribution is reduced sufficiently. Importantly, we analyze the algorithm and show the precise nature of the performance vs. computation tradeoff. Finally, we show the effectiveness of the method on a web server control application.  
\end{abstract}

\input{intro}

\input{alg}

\input{bounded-state}

\input{unbounded-state}

\input{examples}

\input{experiments}

%\input{conclusion}

%\begin{small}
\bibliography{./biblio}
%\end{small}
%\bibliographystyle{unsrt}

\newpage

\appendix 
\input{appendix}

\end{document}

%% file: Commands.tex
\usepackage{pdfsync}
\usepackage{graphicx} % more modern
\newcommand{\todoy}[2][]{\todo[size=\tiny,color=red!20!white,#1]{Yasin: #2}}
\newcommand{\todoc}[2][]{\todo[size=\tiny,color=green!20!white,#1]{Csaba: #2}}

\usepackage{algorithm}
\usepackage{algorithmic}

\usepackage{amssymb}

\usepackage{amsmath} % Learn about the AMS package, again very useful!

\usepackage{color}
\usepackage{graphicx}
\usepackage{epstopdf}
\usepackage{algorithm,algorithmic}
\usepackage{verbatim}

\usepackage{nicefrac}

\newif\ifcomm
\commtrue

\newif\iflong
\longtrue
\newtheorem{thm}{Theorem}
\newtheorem{cor}[thm]{Corollary}
\newtheorem{lem}[thm]{Lemma}

\newtheorem{prop}[thm]{Proposition}

\newcounter{assumption}%[section]
\renewcommand{\theassumption}{A\arabic{assumption}}

\newenvironment{ass}[1][]{\begin{trivlist}\item[] \refstepcounter{assumption}%
 {\bf Assumption\ \theassumption\ {\em (#1)} } }{%\par\nobreak\noindent\sl\ignorespaces}{%
 \ifvmode\smallskip\fi\end{trivlist}}

\newtheorem{example}{Example}[section]

\newcommand{\norm}[1]{\left\Vert#1\right\Vert}
\newcommand{\smallnorm}[1]{\|#1\|}
\newcommand{\abs}[1]{\left\vert#1\right\vert}

\newcommand{\trace}{\mathop{\rm trace}}

\newcommand{\Real}{\mathbb R}                        % Real numbers
\newcommand{\real}{\mathbb R}                        % again..
\newcommand{\Prob}[1]{{\mathbb P}\left(#1\right)}    % Probabilities; example: \Prob{X>\eps}<1-\delta
\newcommand{\EE}[1]{{\mathbb E}\left[#1\right]}      % Expectations
\newcommand{\one}[1]{\mathbf 1 \left\{#1\right\}}           % Characteristic function

\newcommand{\ra}{\rightarrow}

\newcommand{\argmin}{\mathop{\rm argmin}}

\newcommand{\eqdef}{\stackrel{\mbox{\rm\tiny def}}{=}}

\newcommand{\beq}{\begin{equation}}
\newcommand{\eeq}{\end{equation}}
\newcommand{\beqa}{\begin{eqnarray}}
\newcommand{\eeqa}{\end{eqnarray}}
\newcommand{\beqan}{\begin{eqnarray*}}
\newcommand{\eeqan}{\end{eqnarray*}}
\newcommand{\ben}{\begin{eqnarray*}}
\newcommand{\een}{\end{eqnarray*}}
\newcommand{\bea}{\begin{align*}}
\newcommand{\eea}{\end{align*}}

\ifcomm
   \newcommand\comm[1]{\textcolor{blue}{ #1}}
   \newcommand{\mtodo}[2]{\todo{{\bf #1}: #2}} % To add comments into the text; the first argument is "who", the second is "what"
   \def\here#1{{\bf $\langle\langle$#1$\rangle\rangle$}}

\else
   \newcommand\comm[1]{}
   \newcommand{\mtodo}[2]{}
   \def\here#1{}
\fi

\newcommand{\cZ}{{\cal Z}}
\newcommand{\cX}{{\cal X}}

\newcommand{\cA}{{\cal A}}

\newcommand{\cB}{{\cal B}}

\newcommand{\cR}{{\cal R}}

\newcommand{\cF}{{\cal F}}

\renewcommand{\phi}{\varphi}

\newcommand{\cS}{{\cal S}}

\newcommand{\ttop}{^\top}

\newcommand{\cN}{{\cal N}}

\newcommand{\hTh}{\widehat\Theta}
\newcommand{\TTh}{\widetilde\Theta}

\newcommand{\SPD}{\mathbb{S}^+}
\newcommand{\PSD}{\mathbb{S}^+}
\newcommand{\pistab}{\pi_{\mathrm{stab}}}

%% file: intro.tex
\section{Introduction}

The topic of this paper is Bayesian optimal control,
where the problem is to design a policy that achieves optimal performance on the average
over  control problem instances that are randomly sampled from a given distribution. 
This problem naturally arises when the goal is to design a controller for mass-produced systems, where  production is imperfect but the errors follow a regular pattern and the goal is to maintain a good average performance
over the controlled systems, rather than to achieve good performance even for the system with the largest errors.

In a Bayesian setting, the optimal policy (which exists under appropriate regularity conditions) is history dependent.
Given the knowledge of the prior, the transition dynamics and costs, 
the problem in a Bayesian setting is to find an efficient way to \emph{calculate} 
the actions that the optimal policy would take 
given some history.
Apart from a few special cases, the optimal policy is not available analytically and
requires the evaluation of an intractable integral over action sequences \citep{Martin67:book},
forcing one to resort to use some suboptimal policy.
The question is then how well the suboptimal policy performs compared to the optimal (but infeasible) policy.

This question was answered in a certain sense for finite state and action spaces by 
 \citet{Asmuth-Li-Littman-Nouri-Wingate-2009} and \citet{Kolter-Ng-2009}.
Both works propose specific computationally efficient algorithms, 
which are shown to be $\epsilon$-Bayes-optimal with probability $1-\delta$
with the exception of $O(\textrm{poly}(1/\epsilon))$ many steps,
where for both algorithms $\epsilon$ and $\delta$ are both part of the input.
 % Kolter-Ng-2009: 
 % Bayesian Exploration Bonus (BEB),
 % near-Bayes except for finitely many steps
 % PAC-MDP and near-Bayes contradict
 % 
 % Asmuth et al. 99:
 % Best of Sampled Set (BOSS)
 % near-Bayes except for finitely many steps
While \citet{Kolter-Ng-2009} suggest to add an exploration bonus to the rewards while using the mean estimates for the transition probabilities and considers a finite horizon setting, \citet{Asmuth-Li-Littman-Nouri-Wingate-2009} consider discounted total rewards and a variant of posterior sampling,
originally due to \citet{Thompson-1933} and first adapted to reinforcement learning by \citet{Strens-2000}.
More recently, the algorithm of \citet{Strens-2000} was revisited by \citet{Osband-Russo-VanRoy-2013} in the context of episodic, finite MDPs.
Unlike the previously mentioned algorithms, the posterior sampling algorithm requires neither
the target accuracy $\epsilon$, nor the failure probability $\delta$ as its inputs.
Rather, the guarantee presented 
by \citet{Osband-Russo-VanRoy-2013}
 is that the algorithm's regret, i.e., the excess cost due to not following
the optimal policy, is bounded by $\widetilde{O}(\sqrt{T})$%
 \footnote{$\widetilde{O}(\cdot)$ hides poly-logarithmic factors.}
both with high probability and in expectation.
Of course, the above is just a small sample of the algorithms developed for Bayesian reinforcement learning,
but as far as we know, apart from these, there are no works on algorithms that come with theoretical guarantees.
Nevertheless, the reader interested in further algorithms 
may consult the papers of \citet{Vlassis-Ghavamzadeh-Mannor-Poupart-2012} and \citet{Guez-Silver-Dayan-2013}
who give an excellent overview of the literature. 

%\cite[e.g.,][]{DeaFrieRus98,DeaFriDa99,Guez-Silver-Dayan-2013}. 
%\cite{Guez-Silver-Dayan-2014} Example 1: Phases are necessary.

The starting point of our paper is that of \citet{Osband-Russo-VanRoy-2013}.
In particular, just like \citet{Osband-Russo-VanRoy-2013}, 
 we study the posterior sampling algorithm of \citet{Strens-2000}.
However, unlike \citet{Osband-Russo-VanRoy-2013}, we allow the state-action space to be infinite (subject to some regularity conditions discussed later) and we consider the \emph{infinite horizon, continuing setting}.
In fact, our results show that the issue of the cardinality of the state-action space is secondary to whether there exists a compact description of the uncertainty of the controlled system; here we consider the case when the \emph{dynamics is smoothly parameterized in some unknown parameters}.
As proposed by \citet{Strens-2000}, the algorithm works in phases:
At the beginning of each phase, a policy is computed based on solving the optimal control problem for a random parameter vector drawn from the posterior over the parameter vectors. The algorithm keeps the policy until the parameter uncertainty is reduced by a substantial margin, when a new phase begins and the process is repeated.

An important element of the algorithm is that the lengths of phases is adjusted to the uncertainty left.
While in the case of episodic problems the issue of how long a policy should be kept does not arise, 
in a continuing problem with no episodic structure, 
if policies are changed too often, performance will suffer (see, e.g., Example~1 of \citet{Guez-Silver-Dayan-2014}).
To address this challenge, for  non-episodic problems, \citet{Strens-2000} suggested that the lengths of 
phases should be adjusted to the ``planning horizon''
\citep{Strens-2000}, which however, is ill-defined for the average cost setting that we consider in this paper.
The idea of ending a phase when uncertainty is reduced by a significant margin
goes back at least to \citet{Jaksch-Ortner-Auer-2010}.

Our main result shows that under appropriate conditions, 
 the expected regret of our algorithm
 is $\tilde{O}(\sqrt{T}+\Sigma_T)$, where $T$ is the number of time steps and
  $\Sigma_T$ is controlled by the precision with which the optimal control problems are solved, 
  thus providing an explicit bound on the cost of using imprecise calculations.  
%While solving an optimal control problem at the beginning of each phase might be difficult on its own, 
Thus, the main result of the paper shows that 
	 near-optimal Bayesian optimal control is possible as long as we can efficiently sample 
	 from the parameter posteriors and if we can efficiently solve the arising 
	 classical optimal control problems.
Our proofs combine the proof techniques of \citet{Osband-Russo-VanRoy-2013} with that of \citet{Abbasi-Yadkori-Szepesvari-2011}.

%\subsection{Why Bayesian-optimal control?}
%We assume that there are sensible Bayesian optimal control problems. Indeed there are (give examples? E.g., a controller must be produced for a mass produced system where the distribution of  the errors is well understood and we care about the overall performance across all deployed controllers more than the performance of the individual controllers.) [Add to the conclusion: We believe that the results will extend beyond Bayesian optimal control problems.]

%\subsection{Why approximate computations and why regret?}
%In a Bayesian setting, under some regularity conditions the optimal policy would be well-defined and hence learning is reduced to the computational issue of calculating the actions an optimal policy would take given some history. Apart from a few special cases, the optimal policy is not available analytically, hence one needs to resort to approximate computations, resulting in an approximate policy. If the original problem is to minimize the average cost per time step, then the natural measure of the approximate policy's suboptimality is the regret (or average regret). 

%\subsection{Why a reduction?}
%The usual story. 
%Reductions are appealing as they connect problems, allowing improving solutions to the target problems to transfer directly into improving solutions for the reduced problem.

%We study a control problem with a non-linear transition law. 
\section{Problem Setting}

We consider problems when
the transition dynamics is parameterized with a matrix \todoc{Why does it have to be a matrix!!?!?}
$\Theta_*\in \real^{m\times n}$, which  is sampled at time $0$ (before the interaction with the learner starts) 
from a known prior $P_0$ with support $\cS\subset \real^{m\times n}$.
Let $\cX\subset\Real^n$ be the state space and $\cA\subset\Real^d$ be the action space,
	$x_t\in\cX$ be the state at time $t$ and $a_t\in\cA$ be the action at time $t$, which is chosen based on 
	$x_1,a_t,\ldots,a_{t-1},x_t$. 
It is assumed that $x_1$ is sampled from a fixed distribution (although, it should become clear later that this assumption is not necessary).
For $M\succeq 0$ positive semidefinite, define  
$\norm{\Theta}^2_M = \norm{ \Theta^\top M \Theta}_2$, where $\norm{\cdot}_2$ denotes the spectral norm of matrices.
The set of positive semidefinite $m\times m$ matrices will be denoted by $\SPD(m)$.
Our main assumption concerning the transition law is as follows:
\if0
\begin{ass}[Linearly Parameterized Dynamics]
\label{ass:lindyn}
Given $x_t,a_t$ the next state $x_{t+1}$ is chosen from a probability kernel $p(\cdot|x,a,\Theta_*)$
such that for some $p_w(\cdot|x,a)$ probability kernel over $\cX\times \cA$ and 
some ``feature map'' $\phi:\cX\times \cA \ra \real^m$,
\beq
\label{eq:nonlinear-dynamics}
x_{t+1} = \Theta_*\ttop \phi(x_t,a_t) + w_{t+1}\,,
\eeq
where  and $w_{t+1}\sim p_w(\cdot|x_t,a_t)$.
In particular, the distribution of $w_{t+1}$ is independent of $\Theta_*$ given $\cF_t$.
\end{ass}
\fi
\begin{ass}[Smoothly Parameterized Dynamics]
\label{ass:lindyn}
The next state satisfies $x_{t+1} = f(x_t,a_t,\Theta_*,z_{t+1})$, where $z_{t+1}\sim U[0,1]$
is independent of the past and $\Theta_*$.
Further, there exists a (known) map $M: \cX \times \cA \to \SPD(m)$ such that
for any $\Theta,\Theta'\in \cS$, 
if  $y = f(x,a,\Theta,z)$, $y'=f(x,a,\Theta',z)$ with $z\sim U[0,1]$,
then
$\EE{ \norm{y - y'} } \le \norm{\Theta-\Theta'}_{M(x,a)}$.
\todoc{Relaxing this? Taylor series expansion of $f$?}
\end{ass}
\todoc{Note that 
$\EE{ \norm{y - y'}^2 } \le \norm{\Theta-\Theta'}^2_{M(x,a)}$ (which might be easier to prove) is sufficient to meet this assumption. Indeed, if this latter is satisfied, we also have
$\EE{ \norm{y-y'} } \le \EE{ \norm{y-y'}^2}^{\frac12} \le \norm{\Theta-\Theta'}_{M(x,a)}$.
}
The first part of the assumption just states that given $\Theta_*$, the dynamics is Markovian with state $x_t$,
while the second part demands that small changes in the parameter lead to small changes in the next state.
That the map is ``known'' allows us to use it in our algorithms.
Examples of systems that fit our assumptions 
include finite MDPs (where the state is represented by unit vectors) and systems with linear dynamics (i.e., when $x_{t+1} = A x_t + B a_t + w_{t+1}$, where $w_{t+1}\sim p_w(\cdot|x_t,a_t)$).
\todoc{Need to remove ``linear parameterization'' from intro etc.}

The goal is to design a controller (also known as a policy) that at every time step $t$, 
	based on past states $x_1,\ldots,x_t$ and actions $a_1,\ldots,a_{t-1}$, 
	selects an action $a_t$ so as to minimize the expected long-run average loss 
	$\EE{ \limsup_{n\ra\infty} \frac1n\sum_{t=1}^n \ell(x_t,a_t)}$.
We consider any noise distribution and any loss function $\ell$ as long as a boundedness assumption on the variance and a smoothness assumption on the {\em value function} are satisfied (see Assumptions~\ref{ass:bnd-var} and~\ref{ass:ACOE}-\ref{ass:grad} below). 
It is important to note that we allow $\ell$ to be a nonlinear function of the last state-action pair, i.e., the framework allows one to go significantly beyond the scope of linear quadratic control as many nonlinear control problems can be transformed into a linear form (but with a nonlinear loss function) using the so-called dynamic feedback linearization techniques \citep{Isidori95}. \todoc{Follow up in conclusion}

As it was mentioned before, 
computing a (Bayesian) optimal policy $\pi^*$ which achieves the best expected long-run average loss is computationally challenging \citep{Martin67:book}. \todoc{Any hardness results?}
Hence, one needs to resort to approximations.
To measure the suboptimality of an arbitrary (history dependent) policy $\pi$, 
we use the (expected) regret $R_T(\pi)$ of $\pi$:
\[
R_T(\pi) = \EE{ \sum_{t=1}^T \ell(x_t^\pi,a_t^\pi) -  \sum_{t=1}^T \ell(x_t^{\pi^*}, a_t^{\pi^*}) } \,.
\] 
Here, $(x_t^{\pi},a_t^{\pi})_{t=1}^T$ denotes the state-action trajectory that results from following the policy $\pi$.
The slower the regret grows, the closer is the performance of $\pi$ to that of an optimal policy. If the growth rate of $R_T(\pi)$ is sublinear ($R_T(\pi) = o(T))$, the average loss per time step will converge to the optimal average loss as $T$ gets large and in this sense we can say that $\pi$ asymptotically-optimal. Our main result shows that,
under some conditions, the construction of such asymptotically-optimal policies 
can be reduced to the ability of efficiently sampling from the posterior of $\Theta_*$ and being able to
solve classical (non-Bayesian) optimal-control problems.
Furthermore, our main result also implies that $R_T(\pi) = \widetilde{O}(\sqrt{T})$.

%% file: alg.tex
\section{The \textsc{Lazy  PSRL} Algorithm}

%Let $p(.\,|x,a,\Theta)$ be the transition law parametrized by $\Theta$. 

Our algorithm is an instance of the posterior sampling reinforcement learning (PSRL) 
\citep{Osband-Russo-VanRoy-2013}.
As explained beforehand, this algorithm was first proposed by \citet{Strens-2000} in the context
of discounted MDPs and later analyzed by \citet{Osband-Russo-VanRoy-2013} for undiscounted episodic problems.
To emphasize that the algorithm keeps the current policy for a while, we call it \textsc{Lazy PSRL}.
The pseudocode of the algorithm is shown in Figure~\ref{alg:Thompson-nonlin}. 

Recall that $P_0$ denotes the prior distribution of the parameter matrix $\Theta_*$.
Let $P_t$ denote  the posterior of $\Theta_*$ at time $t$ based on $x_1,a_1,\ldots,a_{t-1},x_t$ and  $\tau_t < t$  the last round when the algorithm chose a policy.
Further, let $V_t  = V + \sum_{s=1}^{t-1} M(x_s,a_s)$, where $V$ is an $m\times m$ positive definite matrix.
Then,  at time $t$, Lazy PSRL sets $\widetilde \Theta_t = \widetilde \Theta_{t-1}$ unless 
 $\det(V_t) > 2 \det(V_{\tau_t})$ in which case it chooses $\TTh_t$ from the posterior $P_t$:
$\widetilde \Theta_t\sim P_t$.
The action taken at time step $t$  is
a near-optimal action for the system whose transition dynamics is specified by $\TTh_t$.
We assume that a subroutine, $\pi^*$, taking the current state $x_t$ and the parameter $\TTh_t$ is available to calculate such an action. The inexact nature of calculating a near-optimal action will also be taken in our analysis.

\begin{figure}[tbh]
\begin{center}
\framebox{\parbox{12cm}{
\begin{algorithmic}
\STATE {\bf Inputs}: $P_0$, the prior distribution of $\Theta_*$, $V$.
\STATE $V_{\textrm{last}} \leftarrow V$, $V_0 = V$.
\FOR{$t=1,2,\dots$}
\IF{$\det(V_t)>2 \det(V_{\textrm{last}})$}
\STATE Sample $\TTh_{t}\sim P_t$.
\STATE $V_{\textrm{last}} \leftarrow V_t$.
\ELSE
\STATE $\TTh_{t} \leftarrow \TTh_{t-1}$.
\ENDIF
\STATE Calculate near-optimal action $a_t \leftarrow \pi^*(x_t, \TTh_t)$.
\STATE Execute action $a_t$ and observe the new state $x_{t+1}$.
\STATE Update posterior of $\Theta_*$ with $(x_t,a_t,x_{t+1})$ to obtain $P_{t+1}$.
\STATE Update $V_{t+1} \leftarrow V_t + M(x_t,a_t)$.
\ENDFOR
\end{algorithmic}
}}
\end{center}
\caption{Lazy PSRL for smoothly parameterized control problems}
\label{alg:Thompson-nonlin}
\end{figure}

%% file: bounded-state.tex
\section{Results for Bounded State- and Feature-Spaces}

In this section, we study problems with a bounded state space. The number of states might be infinite, but we assume
that the norm of the state vector is bounded by a constant. Before stating our main result, we state some extra assumptions.
In Section~\ref{sec:examples}, we will show that these assumptions are met for some interesting special cases.
%\subsection{Assumptions}

The first assumption is a restriction on the prior distribution such that a ``variance term'' remains bounded by a constant. 
\todoc{Emphasize in discussion that this assumption replaces the concentration argument of Osband et al.}
\begin{ass}[Concentrating Posterior]
\label{ass:bnd-var}
Let $\tilde{\cF}_t = \sigma(x_1,a_1,\ldots,a_{t-1},x_t)$ be the $\sigma$-algebra generated by observations up to time $t$.
There exists a positive constant $C$ such that for any $t\ge 1$, 
for some $\widehat \Theta_t$ $\tilde{\cF}_t$-measurable random variable,
letting $\Theta_t' \sim P_t$ it holds that
\footnote{We use $\norm{v}$ to denote the $2$-norm of vector $v$.} $\max_{\Theta\in \{\Theta_t',\Theta_*\}}
\EE{\norm{(\Theta - \widehat\Theta_{t})\ttop V_{t}^{1/2}}^2} \le C \; \quad \textrm{a.s.}$ 
\end{ass}
The idea here is that $\widehat \Theta_t$ is an estimate of $\Theta_*$ based on past information available at time $t$,
such as a maximum aposteriori (MAP) estimate.
Since $V_t$ is increasing at a linear rate, the assumption requires 
that $\widehat\Theta_{t}$ converges to $\Theta$ at an $O(1/\sqrt{t})$ rate.
When $\Theta = \Theta_*$, this means that $\widehat\Theta_{t}$ should converge to $\Theta_*$ at this rate, which is indeed what we expect.
When $\Theta = \Theta_t'$, again, we expect this to be true since $\Theta_t'$ is expected to be in the $O(1/\sqrt{t})$ vicinity of $\Theta_*$.
We emphasize that our approach is to reduce the problem to studying the ``variance like terms", as opposed to e.g. \citet{Osband-Russo-VanRoy-2013} who reduce the question directly to a UCB type argument. 

Our next assumption concerns the existence of 
``regular'' solutions to the average cost optimality equations (ACOEs):
\begin{ass}[Existence of Regular ACOE Solutions]
\label{ass:ACOE}
The following hold:
\begin{enumerate}[(i)]
\item \label{ass:ACOE-exist}
There exists $H>0$ such that
for any $\Theta\in \cS$, there exist a scalar $J(\Theta)$ and a function $h(.,\Theta):\cX\ra [0,H]$ 
%bounded from below 
that satisfy the following average cost optimality equation (ACOE) for any $x\in\cX$: 
\beq
\label{eq:ACOE1}
J(\Theta) + h(x,\Theta) =  \min_{a\in \cA} \left\{ \ell(x,a) + \int h(y,\Theta) p(d y\, |\, x,a,\Theta) \right\}\;,
\eeq
where $p(\cdot|x,a,\Theta)$ is the next-state distribution given state $x$, action $a$ and parameter $\Theta$. 
\item \label{ass:grad}
There exists $B>0$ such that for all $\Theta \in \cS$, and for all $x,x^\prime\in \cX$, $\abs{h(x,\Theta)-h(x^\prime,\Theta)}\leq B \norm{x-x^\prime}$. 
\end{enumerate}
%Further, we assume that there exists an oracle that, for any possible value of $\Theta$ in $\cS$, provides the solution of the ACOE.
\end{ass}
%The optimization problem \eqref{eq:ofu} is defined over reachable and observable matrices. The reachability and observability assumptions guarantee that the Riccati equation has a solution, which in turn guarantees a solution to the average cost optimality equation (ACOE). Thus, a solution to the ACOE is guaranteed for a LQ problem whose model is specified by the solution of \eqref{eq:ofu}. Thus, we could write the optimality equation \eqref{eq:BOE}, from which we obtained the regret decomposition. %Notice that it is relatively easy to verify the reachability and observability of a matrix; therefore, the optimistic problem can be solved. 

%The regret analysis in the previous section started by writing the average cost optimality equation (ACOE) for a parameter vector in the confidence set. Existence of a solution to ACOE in the LQ problem follows from existence of a solution to the Riccati equation, which in turn follows from reachability and observability assumptions. Notice that it is relatively easy to verify that a parameter matrix satisfies the reachability and observability assumptions. 

With a slight abuse of the concepts, we will call the quantity $J(\Theta)$ the average loss of the optimal policy, while function $h(.,\Theta)$ will be called the value function (for the system with parameter $\Theta$). 
%In what follows, we denote  $h(.,\Theta_*)$ by $h_*(.)$. 
%Appendix~\ref{app:ACOI} shows a number of cases when Assumption~\ref{ass:ACOE} is satisfied.
The review paper by \citet{Arapostathis-Borkar-et-al-1993} gives a number of sufficient (and sometimes necessary) conditions that guarantee that a solution to ACOE exists. In this paper, we assume that the required conditions are satisfied. 
Lipschitz continuity usually follows from that of the transition dynamics and the losses.
\todoc[size=\tiny]{Will this be satisfied in the examples? Under the Lipschitzness of the cost and the transition dynamics? Which norm is used in the assumption?}

A uniform lower bound on $h$ follows, for example if the immediate cost function $\ell$ is lower bounded.
Then, if the state space is bounded, uniform boundedness of the functions $h(\cdot,\Theta)$ follows from their uniform Lipschitzness:
\todoc{Future work: Extension to $\ell$ depending on $\Theta$. Relaxing uniform boundedness?}
\begin{prop}\label{lem:rangevaluefun}
Assume that the value function $h(\cdot,\Theta)$ is bounded from below ($\inf_x h(x,\Theta)>-\infty$) and is $B$-Lipschitz.
Then, if the diameter of the state space is bounded by $X$ (i.e., $\sup_{x,x'\in \cX}\norm{x-x'} \le X$) then
there exists a solution $h'(\cdot,\Theta)$ to  \eqref{eq:ACOE1} such that the range of $h$ is included in $[0,BX]$.
\todoc[size=\tiny]{LQ problems with bounded state spaces: This is unclear at this point. Return to this later.}
\end{prop}
Proof is in Appendix~\ref{app:proofs}. 
Finally, we assume that the map $M: \cX\times \cA \to \PSD(m)$ is bounded: 
\begin{ass}[Boundedness]
\label{ass:feature-mapping}
There exist $\Phi>0$ such that for all $x\in \cX$ and $a\in\cA$, $\trace(M(x,a))\le \Phi^2$. 
\end{ass} 
This assumption may be strong. In the next section we discuss  an extension of the result of this section to the case when this assumption is not met.

%We can relax this assumption by making appropriate assumptions on the feature mapping. We have decided to use the stronger condition to simplify the presentation. 

%We assume that that noise is sub-Gaussian:
%\begin{ass}\label{ass:subGaussian++}
%The random variables $w_t$ are component-wise $L$-sub-Gaussian.
%\end{ass}

%This assumption is satisfied, for instance, for the linear transition model of Section~\ref{sec:lqr} as long as the state remains bounded, because $h(x,\Theta)=x\ttop P(\Theta)x$ and the solution to Riccati equation, $P(\Theta)$, is bounded.

%At time $t$, we take action $a_t=a(x_t, \TTh_t)$. 

%\todoy[inline]{if we decide to work with the more general case of $x_{t+1}=f_*(x_t,a_t)+w_{t+1}$, then we have two options: (1) Construct a confidence set around $f_*$. For instance use the methods in~\cite{davies-2007}. (2) Use approximation theory and basis functions of increasing size. But then we need to show that the optimal average cost associated with $f_*$ is close to the one associated with $\Theta_{*,t} \phi$, where $\Theta_{*,t} \phi$ is the closest function to $f_*$ in a function class. This does not look trivial.} 

%\subsection{Analysis}

The main theorem of this section bounds the regret of Lazy PSRL.
In this result, we allow the oracle to return an $\sigma_t$-suboptimal action, where $\sigma_t>0$. 
By this, we mean that the action $a_t$ satisfies
\beq
\label{eq:stsuboptimality}
\ell(x_t,a_t)+\int h(y,\widetilde{\Theta}_t) p(dy|x_t,a_t,\widetilde{\Theta}_t) 
\le 
\min_{a\in \cA}
\left\{\ell(x_t,a)+\int h(y,\widetilde{\Theta}_t) p(dy|x_t,a,\widetilde{\Theta}_t) \right\}+\sigma_t\,.
\eeq
\begin{thm}
\label{thm:LQR++}
Assume that  \ref{ass:lindyn}--\ref{ass:feature-mapping} hold for some values of $C,B,X,\Phi>0$. 
Consider Lazy PSRL where in time step $t$, the action chosen is $\sigma_t$-suboptimal.
Then, for any time $T$, the regret of  Lazy PSRL satisfies $R_T = \widetilde O\left( \sqrt{T} \right) + \Sigma_T$, where $\Sigma_T = \sum_{t=1}^T \EE{\sigma_t}$ 
and the constant hidden by $\widetilde{O}(\cdot)$ depends on $V,C,B,X$ and $\Phi$.
\end{thm}
\noindent 
In particular, the theorem implies  that  Lazy PSRL is asymptotically optimal as long as $\sum_{t=1}^T \EE{\sigma_t} = o(T)$
and it is  $O(\epsilon)$-optimal if $\EE{\sigma_t} \le \epsilon$.

%% file: unbounded-state.tex
\section{Forcefully Stabilized Systems}

For some applications, such as robotics, where the state can grow unbounded,
the boundedness assumption (Assumption~\ref{ass:feature-mapping}) is rather problematic.
\todoc{The example is not that apt given that our assumptions will not be satisfied.}
For such systems,
it is common to use a stabilizing controller $\pistab$ 
that is automatically turned on and is kept on as long as the  state vector is ``large''.
The stabilizing controller, however, is expensive (uses a large amount of energy), 
as it is designed to be robust so that it is guaranteed to drive back the state to the safe region 
for all possible systems under consideration.
Hence a good controller should avoid relying on the stabilizing controller.

In this section, we will replace  Assumption~\ref{ass:feature-mapping} with an assumption that a stabilizing controller is available. We will use this controller to override 
the actions coming from our algorithm 
as soon as  the state leaves the (bounded) safe region $\cR\subset \real^n$ until it returns to it.
The pseudocode of the algorithm is shown in Figure~\ref{alg:Thompson-nonlin-2}.

\begin{figure}[tbh]
\begin{center}
\framebox{\parbox{12cm}{
\begin{algorithmic}
\STATE {\bf Inputs}: $P_0$, the prior distribution of $\Theta_*$, $V$, the safe region $\cR\subset \Real^n$.
\STATE Initialize Lazy PSRL with $P_0$ and $V$, $x_1$.
\FOR{$t=1,2,\dots$}
\IF{$x_t\in \cR$}
\STATE Get action $a_t$ from Lazy PSRL
\ELSE
\STATE Get action $a_t$ from $\pistab$
\ENDIF
\STATE Execute action $a_t$ and observe the new state $x_{t+1}$.
\STATE Feed $a_t$ and $x_{t+1}$ to Lazy PSRL.
\ENDFOR
\end{algorithmic}
}}
\end{center}
\caption{Stabilized Lazy PSRL}
\label{alg:Thompson-nonlin-2}
\end{figure}

We assume that the stabilizing controller is effective in the following sense:
\begin{ass}[Effective Stabilizing Controller]
\label{ass:stbl-cont}
There exists $\Phi>0$ such that the following holds:
Pick any $x\in \cR$, $a\in \cA$ and let $x_1',a_1',x_2',a_2',\ldots$ be the sequence of state-action pairs obtained when
from time step two the Markovian stabilizing controller $\pistab$ is applied to the controlled system whose dynamics is given by $\Theta\in \cS$:
$x_1' = x$, $a_1' = a$, $x_{t+1}' \sim p(\cdot|x_t',a_t',\Theta)$, $a_{t+1}' \sim \pistab(\cdot|x_t')$.
Then, $\EE{ \trace( M(x_t',a_t') ) } \le \Phi^2$ for any $t\ge 1$, where $M: \cX\times \cA \to \PSD(m)$ is the map of Assumption~\ref{ass:lindyn} underlying $\{p(\cdot|x,a,\Theta)\}$.
\end{ass}
The assumption is reasonable as it only requires that the trace of $M(x_t',a_t')$ is bounded \emph{in expectation}.
Thus, large spikes, that no controller may prevent, can exist as long as they happen with a sufficiently low probability.

The next theorem shows that Stabilized Lazy PSRL is near Bayes-optimal for the system $p'$ obtained from $p$
by overwriting the action $a$ by the action $\pistab(x)$ if $x$ is outside of the safe region $\cR\subset \Real^n$:
\begin{align*}
p'(dy|x,a,\Theta) =
\begin{cases}
p(dy|x,a,\Theta), & \text{ if } x\in \cR;\\
p(dy|x,\pistab(x),\Theta), & \text{ otherwise}\,.
\end{cases}
\end{align*}
\begin{thm}
\label{thm:LQR++-stabilized}
Consider a parameterized system with the transition probability kernel family $\{p(\cdot| x,a,\Theta)\}_{\Theta\in \cS}$
and let $\pistab: \cX \to \cA$ be a deterministic Markovian controller.
Let the smooth parameterization assumption~\ref{ass:lindyn} hold for $\{p(\cdot|x,a,\Theta)\}$,
the ACOE solution regularity assumption~\ref{ass:ACOE} hold for $\{p'(\cdot|x,a,\Theta)\}$.
Consider running the Stabilized Lazy PSRL 
algorithm of Figure~\ref{alg:Thompson-nonlin-2} on $p(\cdot|x,a,\Theta_*)$
and let the concentration assumption~\ref{ass:bnd-var} hold along the trajectory obtained.
Then, if in addition Assumption~\ref{ass:stbl-cont} holds then
the regret of Stabilized Lazy PSRL against the Bayesian optimal controller of $\{p'(\cdot|x,a,\Theta)\}_{\Theta}$
with prior $P_0$ and immediate cost $\ell$
satisfies $R_T = \widetilde O\left( \sqrt{T} \right) + \Sigma_T$, where $\Sigma_T = \sum_{t=1}^T \EE{\one{x_t\in \cR} \sigma_t}$ and $\sigma_t$ is the suboptimality of the action computed by Lazy PSRL at time step $t$.
\end{thm}

%This means the feature mapping has the shape of 
%\begin{align*}
%\Theta_*\ttop &=
%\begin{pmatrix}
%M_1\,, & M_2
%\end{pmatrix}
%\qquad \mathrm{ and } 
%\qquad
%\phi(x,a) = 
%\begin{pmatrix}
%\one{\norm{x}\le X} \phi_1(x,a) \\ \one{\norm{x}> X} \phi_2(x,a)
%\end{pmatrix} \,.
%\end{align*}
%The $M_1$ component of parameter matrix corresponds to the dynamics in the normal regime, while $M_2$ corresponds to the dynamic when the stabilizing controller is activated. 

%% file: examples.tex
\section{Illustration}
\label{sec:examples}
The purpose of this section is to illustrate the results obtained.
In particular, we will consider applying the results to finite MDPs and linearly parameterized controlled systems.
\todoc{Ideally, this section states two results (as corollaries). The proofs should amount to verifying the conditions of the theorems. I have copied here stuff, but I did not attempt to do this.}

\subsection{Near Bayes-optimal Control in Finite MDPs}

Consider an MDP problem with finite state and action spaces. Let the state space be $\cX=\{1,2,\dots,n\}$ and the action space be $\cA=\{1,2,\dots,d\}$. We represent the state variable by an $n$-dimensional binary vector $x_t$ that has only one non-zero element at the current state. The feature mapping for a state-action pair is an $nd$-dimensional binary vector that, similarly to the state vector, has only one non-zero element indicating a state-action pair. In particular, the feature mapping and the parameter matrix have the following form: 
\[
\mbox{For } 1\le k \le n d, \quad
\phi_k(x,a)=
\begin{cases}
1, & \text{ if }  k = (a-1)n+x\,;\\
0, & \text{otherwise}\,,
\end{cases}
\qquad
\Theta_* = \begin{pmatrix}
  \Theta_{*}^{(1)}\\
  \Theta_{*}^{(2)} \\
  \vdots \\
  \Theta_{*}^{(d)} 
\end{pmatrix}\in\Real_{+}^{d n \times n}\; .
\]
Let $s\in [n]$ be a state and $a\in [d]$ be an action. The $s$th row of matrix $\Theta_{*}^{(a)}$ is a distribution over the state space that shows the transition probabilities when we take action $a$ in state $s$. Thus, any row of $\Theta_{*}^{(a)}$ sums to one and $\EE{x_{t+1} | x_t,a_t} = \Theta_*^\top \phi (x_t,a_t)$. An appropriate prior for each row is a Dirichlet distribution. Let $\alpha_1,\dots,\alpha_n$ be positive numbers and let $V' = \mathrm{diag}(\alpha_1, \ldots, \alpha_n)$. Then $V = \mathrm{diag}(V',\ldots,V') \in \real^{nd \times nd}$ be our ``smoother''.

 Let the prior for the $s$th row of $\Theta_{*}^{(a)}$ be the Dirichlet distribution with parameters $(\alpha_1,\ldots,\alpha_n)$:
 $(P_0)_{s,:}=D(\alpha_1, \dots, \alpha_n)$.  
 At time $t$, the posterior has the form $(P_t)_{s,:}=D(\alpha_1+c_t(s,a,1), \dots, \alpha_n+c_t(s,a,n))$, where $c_t(s,a,s')$ is the number of observed transitions to state $s'$ after taking action $a$ in state $s$ during the first $t$ time steps. Matrix $V_t$ is a diagonal matrix with diagonal elements showing the number of times a state-action pair is observed. In particular, $(V_{t})_{n(a-1)+s,n(a-1)+s} = \sum_{s'}\, (\alpha_{s'} + c_t(s,a,s') )$. Vector $\widehat \Theta_{t,(:,s')}$ is an $nd$-dimensional vector and its elements show the empirical frequency of transition to state $s'$ from different state-action pairs. The mean of distribution $(P_t)_{s,:}$ is vector $\widehat \Theta_{t,(n(a-1)+s,:)}$ and we have $\widehat \Theta_{t,(n(a-1)+s,s')} = \alpha_{s'}+c_t(s,a,s')/\sum_{s''} (\alpha_{s''}+c_t(s,a,s''))$. 
 
The next corollary shows the performance of Lazy PSRL when applied to a finite MDP problem. The proof is in Appendix~\ref{app:proofs}.
\begin{cor}
\label{cor:LQR++-finite}
Consider Lazy PSRL applied to a finite MDP problem with $n$ states, $d$ actions, and a Dirichlet prior as specified above. Suppose in time step $t$, the action chosen is $\sigma_t$-suboptimal. Then, for any time $T$, the regret of Lazy PSRL satisfies $R_T = \widetilde O\left( \sqrt{T} \right) + \Sigma_T$. %, where the constant hidden by $\widetilde{O}(\cdot)$ depends on $V,n$, and $d$.
\end{cor}

\subsection{Linearly Parametrized Problems with Gaussian Noise}
\label{sec:gaussian}

Next, we consider linearly parametrized problems with Gaussian noise:
\beq
\label{eq:nonlinear-dynamics}
x_{t+1} = \Theta_*^\top \phi(x_t, a_t) + w_{t+1} \,,
\eeq
where $w_{t+1}$ is a zero-mean normal random variable. The nonlinear dynamics shown in \eqref{eq:nonlinear-dynamics} shares similarities to, but allows significantly generality than the Linear Quadratic (LQ) problem considered by~\citet{Abbasi-Yadkori-Szepesvari-2011}. 
In particular, in the LQ problem,
$
\Theta_*\ttop =
\begin{pmatrix}
A_*\,,  B_*
\end{pmatrix}
$
and 
%\qquad 
%\mathrm{ and } 
%\qquad
$
\phi(x_t,a_t)^\top = 
\begin{pmatrix}
x_t^\top\,, a_t^\top
\end{pmatrix}
$. 
Further, \citet{Abbasi-Yadkori-Szepesvari-2011} assume that the noise is subgaussian. 

%More generally, robotic problems also nearly fit the framework. Let $k$ be the number of degrees of freedom of the robot. Let $q$, $\dot{q}$, $\ddot{q}$, $\tau_g$ and $\tau$ be the joint position, velocity, acceleration, gravity term and force vectors, respectively. The canonical equations of motion of a robot mechanism is given by $H(q) \ddot{q} + C(q,\dot{q})\dot{q} + \tau_g(q) = \tau$, where $H$ and $C$ are $k\times k$ matrices~\cite{Siciliano-Khatib-2008}. By defining appropriate features, considering a discrete time approximation and allowing noisy transitions, with an appropriate function $\phi$, one gets that the state-action sequence is subject to the constraint $\EE{\Theta_*\ttop \phi(x_t,a_t,x_{t+1}) \, \middle| \, \cF_t, \Theta_* } = 0$, where $\Theta_*$ collects unknown parameters of the robot (masses and geometric parameters). While this is different from \eqref{eq:nonlinear-dynamics}, we conjecture that our results can be extended to this case. \todoc[size=\tiny]{Where do we argue?}\todoy[size=\tiny]{Text is changed a bit.}

Next, we describe a conjugate prior under the assumption that the noise is Gaussian. Without loss of generality, we assume that $\EE{w_{t+1} w_{t+1}^\top\ |\ \cF_t} = I$. A conjugate prior is appealing 
as the posterior has a compact representation that allows for computationally efficient sampling methods. \todoc[size=\tiny]{There is a bit of a problem when we use a stabilizing controller, or not?}
Assume that the columns of matrix $\Theta_*$ are independently sampled from the following prior: 
$
\mbox{for } i=1\dots n: \quad P_0\left(\Theta_{*,(:,i)}\right) \propto  \exp\left( \Theta_{*,(:,i)}^\top V \Theta_{*,(:,i)} \right) \one{\Theta_{*,(:,i)}\in \cS} 
$. 
Then, by Bayes' rule, the posterior for column $i$ of $\Theta_*$ has the form of 
\[
P_t\left(\Theta_{*,(:,i)} \right) \propto \exp\left( -\frac{1}{2} \left(\Theta_{*,(:,i)} - \widehat \Theta_{t,(:,i)} \right)^\top V_t \left(\Theta_{*,(:,i)} - \widehat \Theta_{t,(:,i)} \right) \right) \one{\Theta_{*,(:,i)}\in \cS} \; .
\]
The next corollary shows the performance of Lazy PSRL when applied to linearly parametrized problems with Gaussian noise. We assume an effective stabilizing controller is available. The proof is in Appendix~\ref{app:proofs}.
\begin{cor}
\label{cor:LQR++-linearly}
Consider Stabilized Lazy PSRL applied to a linearly parametrized problem with Gaussian noise. Suppose in time step $t$, the action chosen is $\sigma_t$-suboptimal. Then, for any time $T$, the regret of Stabilized Lazy PSRL satisfies $R_T = \widetilde O\left( \sqrt{T} \right) + \Sigma_T$. %, where the constant hidden by $\widetilde{O}(\cdot)$ depends on $V,n$, and $d$.
\end{cor}

%% file: experiments.tex
\section{Experiments}

The purpose of this section is to illustrate the behavior of Lazy PSRL on a simple control problem.
As the control problem, we choose a web server control problem.
This control problem is described first, which will be followed by the description of our results.

\subsubsection{Web Server Control Application}
Next, we illustrate the behavior of Lazy PSRL on a web server control problem.
The problem is taken from~Section 7.8.1 of the book by \citet{HeDiPaTi04} (this example is also used in Section 3.4 of the book by \citet{Astrom-Murray-2008}). 
An Apache HTTP web server processes the incoming connections that arrive on a queue. Each connection is assigned to an available process. A process drops the connection if no requests have been received in the last \textsc{KeepAlive} seconds. At any given time, there are at most \textsc{MaxClients} active processes. The values of the \textsc{KeepAlive} and \textsc{MaxClients} parameters, denoted by $a_{ka}$ and $a_{mc}$ respectively, are chosen by a control algorithm. Increasing $a_{mc}$ and $a_{ka}$ results in faster and longer services to the connections, but also increases the CPU and memory usage of the server. \textsc{MaxClients} is bounded in $[1,20]$, while \textsc{KeepAlive} is bounded in $[1,1024]$. The state of the server is determined by the average processor load $x_{cpu}\in [0,1]$ and the relative memory usage $x_{mem}\in [0,1]$. A  {\em operating point of interest} of the system is given by $x_{cpu} = 0.58\,, a_{ka} = 11 s\,, x_{mem} = 0.55\,, a_{mc} = 600$. A linear model around the operating point is assumed, resulting in a model of the form
\[
\begin{pmatrix}
  x_{cpu}(t+1)\\
  x_{mem}(t+1) 
\end{pmatrix} = 
\begin{pmatrix}
  A_{11} & A_{12}\\
  A_{21} & A_{21}
\end{pmatrix}\,
\begin{pmatrix}
  x_{cpu}(t)\\
  x_{mem}(t) 
\end{pmatrix} + 
\begin{pmatrix}
  B_{11} & B_{12}\\
  B_{21} & B_{21}
\end{pmatrix}\,
\begin{pmatrix}
  a_{ka}(t)\\
  a_{mc}(t) 
\end{pmatrix} +
\begin{pmatrix}
  w_1(t+1)\\
  w_2(t+1) 
\end{pmatrix}
\,,
\]
where $(w_1(t+1),w_2(t+1))_t$ is an i.i.d. sequence of Gaussian random variables, with a diagonal covariance matrix $\EE{w(t+1)^\top w(t+1)} = \sigma^2 I$. We test $\sigma=0.1$ and $\sigma=1.0$ in our experiments. 
Note that these state and action variables are in fact the deviations from the operating point. 
\citet{HeDiPaTi04} fitted this model to an Apache HTTP server and obtained the parameters
\[
A = \begin{pmatrix}
  0.54 & -0.11\\
  -0.026 & 0.63
\end{pmatrix},\qquad
B = \begin{pmatrix}
  -85 & 4.4\\
  -2.5 & 2.8
\end{pmatrix}\times 10^{-4}\;,
\]
while the noise standard deviation was measured to be $0.1$.
\citet{HeDiPaTi04} found that these parameters provided a reasonable fit to their data.

For control purpose, the following cost matrices were chosen (cf. Example~6.9 of \citet{Astrom-Murray-2008}):
\[
Q = \begin{pmatrix}
  5 & 0\\
  0 & 1
\end{pmatrix},\qquad
R = \begin{pmatrix}
  1/50^2 & 0\\
  0 & 0.1^6
\end{pmatrix} \;.
\]
 
\subsubsection{Numerical Results}

We compare the \textsc{Lazy PSRL} algorithm with the \textsc{OFULQ} algorithm~\citep{Abbasi-Yadkori-2012} on this problem. For the \textsc{Lazy PSRL} algorithm, we use standard normal distribution as prior. The \textsc{OFULQ} algorithm is an optimistic algorithm that maintains a confidence ellipsoid around the unknown parameter and, in each round, finds the parameter and the corresponding policy that attains the smallest average loss. Specifically, the algorithm solves optimization problem $(\widetilde A, \widetilde B) = \argmin_{A,B} J(A,B)$, where $J(A,B)$ is the average loss of the optimal policy when system dynamics is $(A,B)$.  Then, the algorithm plays action $a_t = - K(\widetilde A, \widetilde B) x_t$, where $K$ is the \textit{gain matrix}. The objective function $J$ is not convex and thus, solving the optimistic optimization can be very time consuming. As we show next, the \textsc{Lazy PSRL} algorithm can have lower regret while avoiding the high computational costs of the \textsc{OFULQ} algorithm.  

The time horizon in these experiments is $T=1,000$. We repeat each experiment $10$ times and report the mean and the standard deviation of the observations. 
Figure~\ref{figure:exps} shows regret vs. computation time. The horizontal axis shows the amount of time (in seconds) that the algorithm spends to process $T=1,000$ rounds. We change the computation time by changing how frequent an algorithm updates its policy. 
%In these experiments, a single run of the \textsc{Lazy PSRL} algorithm is on average more than eight times faster than the \textsc{OFULQ} algorithm. 
%T=10000: a single run takes, on average, $91.38$ seconds with the \textsc{OFULQ} algorithm and $30.854$ seconds with the \textsc{Lazy PSRL} algorithm. 
Details of the implementation of the \textsc{OFULQ} algorithm are in~\citep{Abbasi-Yadkori-2012}. 

The top-left and right subfigures of Figure~\ref{figure:exps} show the regret of the algorithms when the standard deviation of the noise is $\sigma=0.1$. The regret of the \textsc{Lazy PSRL} algorithm is slightly worse than what we get for the \textsc{OFULQ} algorithm in this case. The \textsc{Lazy PSRL} algorithm outperforms the \textsc{OFULQ} algorithm when the noise variance is larger (bottom subfigures).  %We explain this observation by noting that, in this problem, we need large inputs to reliably estimate matrix $B$, which has small elements. 
We explain this observation by noting that a larger noise variance implies larger confidence ellipsoids, which results in more difficult OFU optimization problems. Finally, we performed experiments with different prior distributions. Figure~\ref{figure:prior} shows regret of the \textsc{Lazy PSRL} algorithm when we change the prior. 

\begin{figure}
\centering%
\begin{tabular}{c}
%\vspace*{-15mm}
%\hspace*{-4mm}
{\includegraphics[width=.35\textwidth]{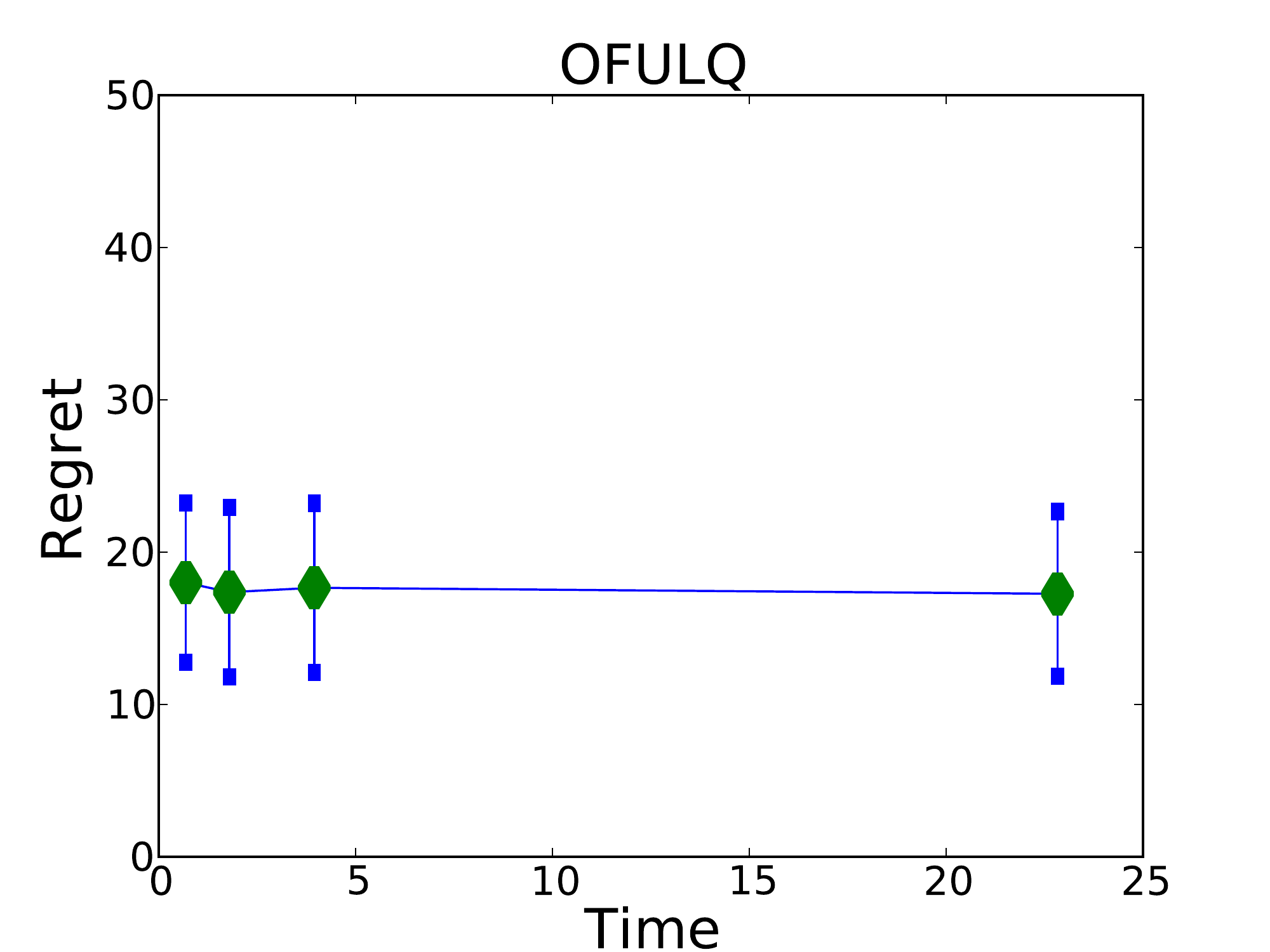}} \hspace{5mm}
{\includegraphics[width=.35\textwidth]{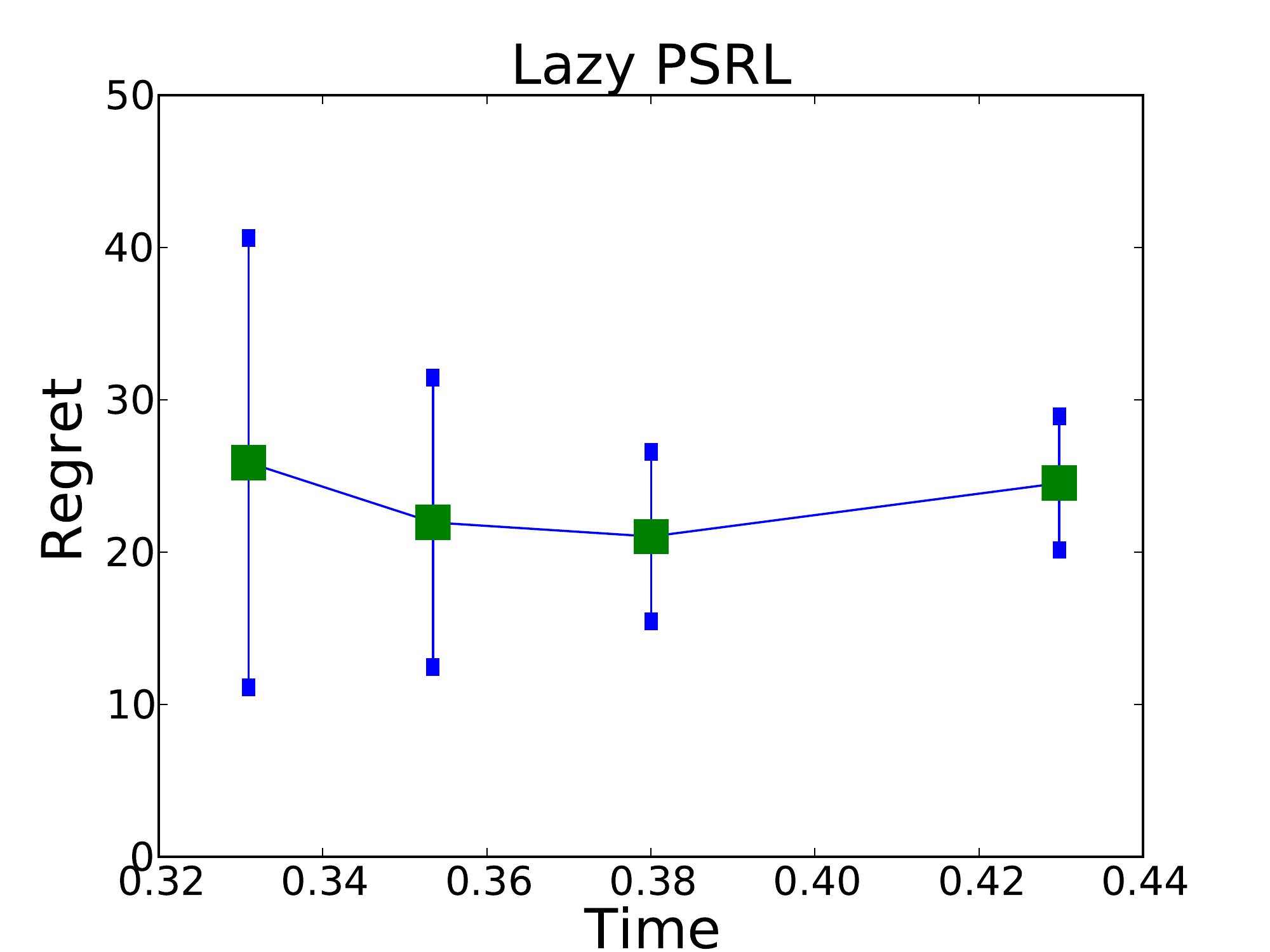}} \\
{\includegraphics[width=.35\textwidth]{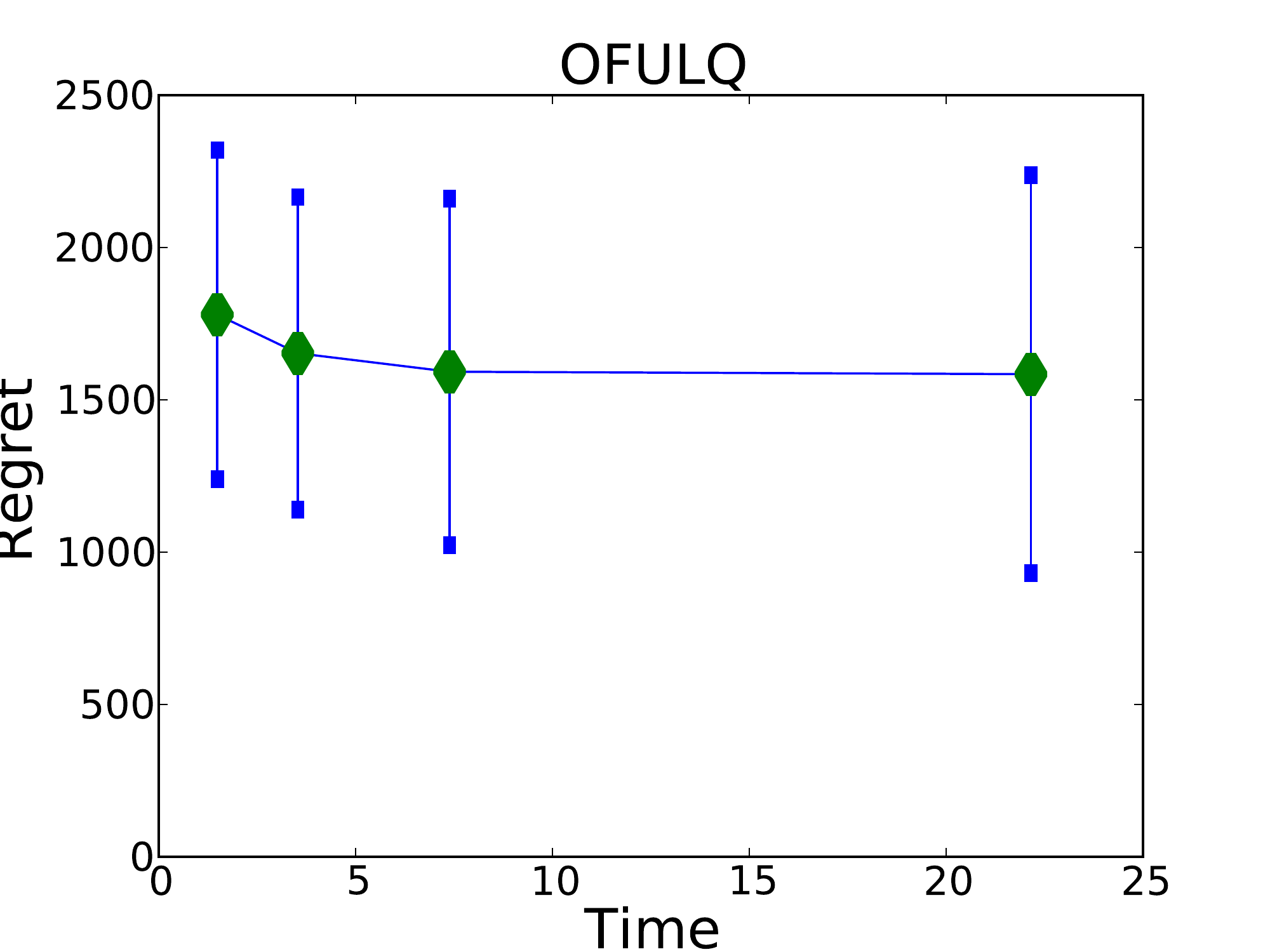}} \hspace{5mm}
{\includegraphics[width=.35\textwidth]{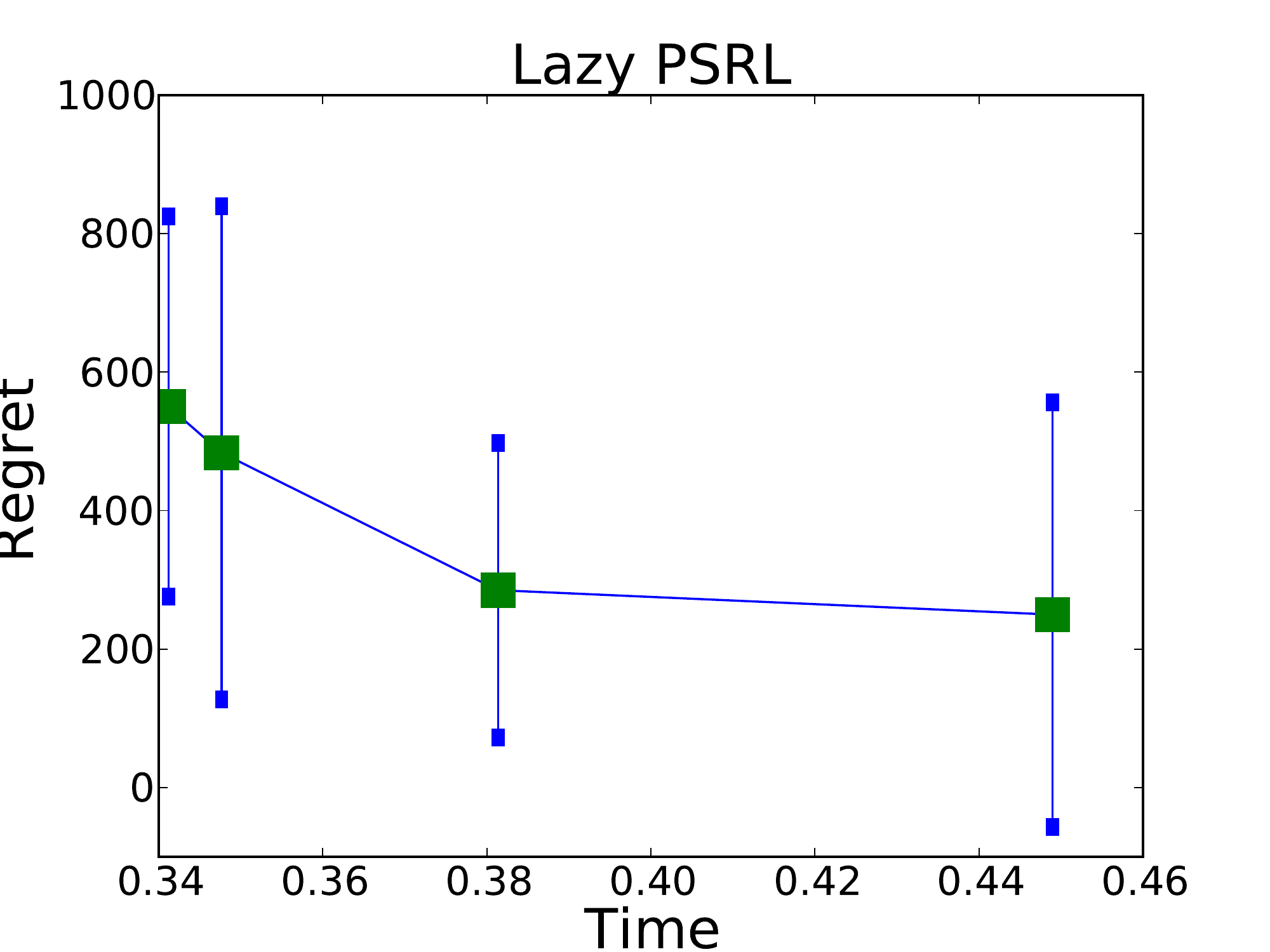}}
%\vspace*{-15mm}
\end{tabular}
\caption{Regret vs time for a web server control problem. (Top-left) regret of the \textsc{OFULQ} algorithm when $\sigma = 0.1$. (Top-right): regret of the \textsc{Lazy PSRL} algorithm when $\sigma = 0.1$. (Bottom-left) regret of the \textsc{OFULQ} algorithm when $\sigma = 1.0$. (Bottom-right): regret of the \textsc{Lazy PSRL} algorithm when $\sigma = 1.0$.}
\label{figure:exps}
\end{figure}

\begin{figure}
\centering%
\begin{tabular}{c}
%\vspace*{-15mm}
%\hspace*{-4mm}
{\includegraphics[width=.35\textwidth]{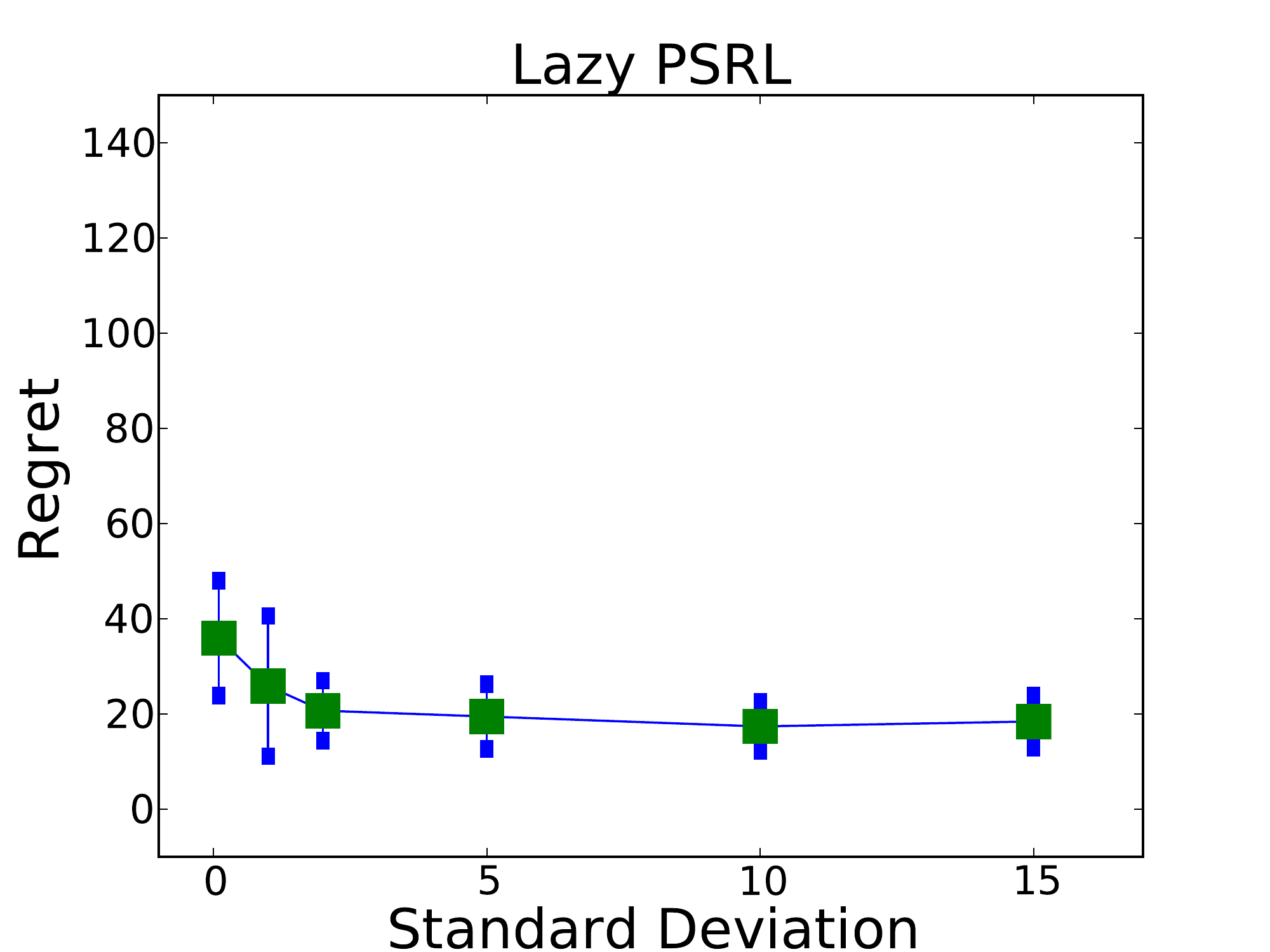}}%\vspace*{-15mm}
\end{tabular}
\caption{Regret of the \textsc{Lazy PSRL} algorithm with different priors. The prior is a zero mean Gaussian distribution with covariance matrix $\lambda^2 I$. The horizontal axis is $\lambda$.}
\label{figure:prior}
\end{figure}

\if0

\begin{figure}
\centering%
\begin{tabular}{c}
%\vspace*{-15mm}
%\hspace*{-4mm}
{\includegraphics[width=.5\textwidth]{./experiments/x1.pdf}} 
{\includegraphics[width=.5\textwidth]{./experiments/x2.pdf}} \\
{\includegraphics[width=.5\textwidth]{./experiments/a1.pdf}}
{\includegraphics[width=.5\textwidth]{./experiments/a2.pdf}}
\\
%\vspace*{-15mm}
\end{tabular}
\caption{The trajectory of the state and action vectors. (Top left): $x_{cpu}$ vs. time. (Top right): $x_{mem}$ vs. time. (Bottom left): $a_{ka}$ vs. time. (Bottom right): $a_{mc}$ vs. time.}
\label{figure:lqr-xu}
\end{figure}

\begin{figure}
\centering%
\begin{tabular}{c}
%\vspace*{-15mm}
%\hspace*{-4mm}
{\includegraphics[width=.5\textwidth]{./experiments/A11.pdf}} 
{\includegraphics[width=.5\textwidth]{./experiments/A12.pdf}} \\
{\includegraphics[width=.5\textwidth]{./experiments/A21.pdf}}
{\includegraphics[width=.5\textwidth]{./experiments/A22.pdf}}
\\
%\vspace*{-15mm}
\end{tabular}
\caption{The least-squares estimate for matrix $A$ vs time.}
\label{figure:lqr-A}
\end{figure}

\begin{figure}
\centering%
\begin{tabular}{c}
%\vspace*{-15mm}
%\hspace*{-4mm}
{\includegraphics[width=.5\textwidth]{./experiments/B11.pdf}} 
{\includegraphics[width=.5\textwidth]{./experiments/B12.pdf}} \\
{\includegraphics[width=.5\textwidth]{./experiments/B21.pdf}}
{\includegraphics[width=.5\textwidth]{./experiments/B22.pdf}}
\\
%\vspace*{-15mm}
\end{tabular}
\caption{Least-squares estimate for matrix $B$ vs time.}
\label{figure:lqr-B}
\end{figure}

\fi

%% file: appendix.tex
\section{Some Useful Lemmas}
\label{app:lem}

\begin{lem}%[Lemma~E.3 of \cite{Abbasi-Yadkori-2012}]
\label{lem:E3}
Let $V\in \PSD(m)$ be positive definite, $(M_t)_{t=1,2,\ldots}\subset \PSD(m)$ be positive semidefinite matrices
and define $V_t = V + \sum_{k=1}^{t-1} M_s$, $t=1,2,\ldots$.
If $\trace(M_t) \le L^2$ for all $t$, then 
\begin{align*}
\sum_{t=1}^T\min(1, \smallnorm{V_t^{-1/2}}_{M_t}^2 )
& \le %\sum_{k=1}^{t-1} \min\left\{1, \normm{a_k}{V_{k}^{-1}}^2 \right\} &\le 
2\left\{\log \det(V_{T+1}) - \log \det V\right\}\\ 
&\le 2\left\{m \log\left(\frac{\trace(V)+T L^2}{m}\right) - \log \det V\right\} \; .
\end{align*}
\end{lem}
\begin{proof}
On the one hand, we have
\begin{align*}
\det(V_T) &= \det( V_{T-1} + M_{T-1} ) = \det( V_{T-1} ( I + V_{T-1}^{-\frac12} M_{T-1} V_{T-1}^{-\frac12}) ) \\
& = \det( V_{T-1} ) \det( I + V_{T-1}^{-\frac12} M_{T-1} V_{T-1}^{-\frac12}) \\
& \vdots \\
& = \det( V ) \prod_{t=1}^{T-1} \det( I + V_{t}^{-\frac12} M_{t} V_{t}^{-\frac12}) \,.
\end{align*}
One the other hand, thanks to $x \le 2 \log(1+x)$, which holds for all $x\in [0,1]$, 
\begin{align*}
\sum_{t=1}^T \min(1, \smallnorm{V_t^{-\frac12} M_t V_t^{-\frac12} }_2 ) 
& \le 2 \sum_{t=1}^T \log( 1 + \smallnorm{V_t^{-\frac12} M_t V_t^{-\frac12} }_2 ) \\
& \le 2 \sum_{t=1}^T \log( \det( I + V_{t}^{-\frac12} M_{t} V_{t}^{-\frac12}) ) \\
& = 2 ( \log (\det V_{T+1}) - \log( \det V ) )\,,
\end{align*}
where the second inequality follows since $V_{t}^{-\frac12} M_{t} V_{t}^{-\frac12}$ is positive semidefinite, hence
all eigenvalues of $I+V_{t}^{-\frac12} M_{t} V_{t}^{-\frac12}$ are above one
and the largest eigenvalue of $I+V_{t}^{-\frac12} M_{t} V_{t}^{-\frac12}$ is $1 + \smallnorm{V_t^{-\frac12} M_t V_t^{-\frac12} }_2$, proving the first inequality.
For the second inequality, note that for any positive definite matrix $S\in \PSD(m)$,
$\log \det S \le m \log( \trace(S)/m)$. Applying this to $V_T$ and using the condition that $\trace(M_t)\le L^2$, 
we get $\log \det V_T \le m \log( (\trace(V)+T L^2)/m )$. Plugging this into the previous upper bound, we get the second part of the statement.
\end{proof}
\begin{lem}[Lemma~11 of \citet{Abbasi-Yadkori-Szepesvari-2011}] %[Lemma 5.14 of \cite{Abbasi-Yadkori-2012}]
\label{lem:514}
Let $A\in \Real^{m\times m}$ and $B\in \Real^{m\times m}$ be positive semi-definite matrices such that $A\succ B$. Then, we have 
\[
\sup_{X\neq 0} \frac{\norm{X^\top A X}_2}{\norm{X^\top B X}_2 } \le \frac{\det(A)}{\det(B)} \; .
\] 
\end{lem}

%%%%%%%%%%%%%%%%%%%%%%%%%%%%%%%%%%%%%%%%%%%%%%%%%%%%%%%%%%%%

\section{Proofs}
\label{app:proofs}

\begin{proof}[Proof of Proposition~\ref{lem:rangevaluefun}]
%We use the Lipschitz continuity and the boundedness assumptions to show that the value function is bounded. 
Note that if ACOE \eqref{eq:ACOE1} holds for $h$, then for any constant $C$, it also holds that 
\[
J(\Theta) + (h(x,\Theta) + C) =  \min_{a\in \cA} \left\{ \ell(x,a) + \int (h(y,\Theta) + C) p(d y\, |\, x,a,\Theta) \right\} \;.
\]
As by our assumption, the value function is bounded from below, 
we can choose $C$ such that the $h'(\cdot,\Theta) = h(\cdot,\Theta)+C$ is nonnegative valued.
In fact, if $h$ assumes a minimizer $x_0$, by this reasoning, without loss of generality,
we can assume that $h(x_0)= 0$ and so for any $x\in \cX$, 
$0\le h(x) = h(x)-h(x_0) \le B \norm{x-x_0}\le BX$. 
The argument trivially extends to the general case when $h$ may fail to have a minimizer over $\cX$. 
\end{proof}

\begin{proof}[Proof of Theorem~\ref{thm:LQR++}]
The proof follows that of the main result of \citet{Abbasi-Yadkori-Szepesvari-2011}.
First, we decompose the regret into a number of terms, which are then bound one by one. 
Define $\widetilde{x}_{t+1}^a=f(x_t,a,\TTh_t,z_{t+1})$, where $f$ is the map of Assumption~\ref{ass:lindyn}
 and let $h_t(x) = h(x, \TTh_t)$
be  the solution of the ACOE underlying $p(\cdot|x,a,\TTh_t)$. 
By Assumption~\ref{ass:ACOE}~\eqref{ass:ACOE-exist}, $h_t$ exists
and $h_t(x)\in [0,H]$ for any $x\in \cX$.
By Assumption~\ref{ass:lindyn}, for any $g\in L^1(p(\cdot|x_t,a,\TTh_t))$, 
$\int g(dy) p(dy|x_t,a,\TTh_t) = \EE{ g(\widetilde{x}_{t+1}^{a}) | \cF_t,\TTh_t}$.
Hence, from~\eqref{eq:ACOE1} and~\eqref{eq:stsuboptimality},
\begin{align*}
J(\TTh_t) + h_t(x_t) &= \min_{a\in \cA} \left\{ \ell(x_t,a) + \EE{h_t(\widetilde{x}_{t+1}^a)\,|\,\cF_t, \TTh_t} \right\} \\
&\ge \ell(x_t,a_t) + \EE{h_t(\widetilde{x}_{t+1}^{a_t})\,|\,\cF_t, \TTh_t} - \sigma_t\\
&= \ell(x_t,a_t) + \EE{h_t( x_{t+1}+\epsilon_t)\,|\,\cF_t, \TTh_t} - \sigma_t\,,
\end{align*}
where $\epsilon_t=\widetilde{x}_{t+1}^{a_t} -x_{t+1}$.
As $J(\cdot)$ is a deterministic function and conditioned on $\cF_{\tau_t}$, $\TTh_t$ and $\Theta_*$ have the same distribution, %\footnote{We have learned this from  \cite{Russo-VanRoy-2013}.} we get that
\begin{align*}
R(T) &= \sum_{t=1}^T \EE{\ell(x_t,a_t) - J(\Theta_*)}= \sum_{t=1}^T \EE{\EE{\ell(x_t,a_t) - J(\Theta_*)\,|\, \cF_{\tau_t} }}\\
&=  \sum_{t=1}^T \EE{\EE{\ell(x_t,a_t) - J(\TTh_t)\,|\, \cF_{\tau_t} }} =  \sum_{t=1}^T \EE{\ell(x_t,a_t) - J(\TTh_t)} \\
&\le \sum_{t=1}^T \EE{h_t(x_t) - \EE{h_t(x_{t+1}+\epsilon_t)\,|\, \cF_t,\TTh_t}} + \sum_{t=1}^T \EE{\sigma_t}\\
&= \sum_{t=1}^T \EE{h_t(x_t) - h_t(x_{t+1}+\epsilon_t)} + \sum_{t=1}^T \EE{\sigma_t}\;.
\end{align*}
Let $\Sigma_T = \sum_{t=1}^T \EE{\sigma_t}$ be the total error due to the approximate optimal control oracle.
Thus, we can bound the regret using 
\begin{align*}
R(T) &\le \Sigma_T + \EE{h_1(x_1) - h_{T+1}(x_{T+1})} + \sum_{t=1}^T \EE{h_{t+1}(x_{t+1}) - h_t(x_{t+1}+\epsilon_t)}\\
&\le \Sigma_T + H + \sum_{t=1}^T \EE{h_{t+1}(x_{t+1}) - h_t(x_{t+1}+\epsilon_t)}\;,
\end{align*}
where the second inequality follows because $h_1(x_1)\le H$ and $-h_{T+1}(x_{T+1})\le 0$.
Let $A_t$ denote the event that the algorithm has changed its policy at time t. We can write
\begin{align*}
%\MoveEqLeft
R(T) - ( \Sigma_T + H) 
&\leq  \sum_{t=1}^T \EE{h_{t+1}(x_{t+1}) - h_t(x_{t+1}+\epsilon_t)}\\
&=  \sum_{t=1}^T \EE{h_{t+1}(x_{t+1}) - h_t(x_{t+1})}+\sum_{t=1}^T \EE{h_{t}(x_{t+1})-h_t(x_{t+1}+\epsilon_t)}\\
%&\leq  2B X \sum_{t=1}^T \EE{\one{A_t}} + \sum_{t=1}^T \EE{h_{t}(x_{t+1})-\EE{ h_t(x_{t+1}+\epsilon_t)\,|\,\cF_t}}\\ 
&\leq   2 H \sum_{t=1}^T \EE{\one{A_t}} + B \sum_{t=1}^T \EE{\norm{\epsilon_t}}\;,
%&=  2 B X \sum_{t=1}^T \EE{\one{A_t}} + B \sum_{t=1}^T \EE{\norm{\epsilon_t}} \; ,
%&\leq 2B X^{b}\sum_{t=0}^T \one{A_t} + B H +\sum_{t=0}^T (h_{t}(x_{t+1})-\EE{ h_t(x_{t+1})~|~\cF_t})~.
\end{align*}
where we used again that $0\le h_t(x)\le H$, and also Assumption~\ref{ass:ACOE}~\eqref{ass:grad}. 
Define
\[
R_1 = H \sum_{t=1}^T \EE{\one{A_t}}\,, \qquad R_2 = B \sum_{t=1}^T \EE{\norm{\epsilon_t}}\; .
\]
It remains to bound $R_2$ and to show that the number of switches is small.

\paragraph{Bounding $R_2$}

Let $\tau_t\le t$ be the last round before time step $t$ when the policy is changed. So $\TTh_t = \TTh_{\tau_t}$.
Letting $M_t = M(x_t,a_t)$, by Assumption~\ref{ass:lindyn},
\[
\EE{ \norm{\epsilon_t} } \le  \EE{ \norm{ \TTh_t - \Theta_*}_{M_t} }.
\]
Further,
\[
\norm{ \TTh_t - \Theta_*}_{M_t} 
\le  \norm{ \TTh_t - \hTh_t}_{M_t}  + \norm{\hTh_t - \Theta_*}_{M_t}\,.
\]
For $\Theta\in\{\TTh_{\tau_t}, \Theta_*\}$ we have that
\begin{align*}
\norm{ \Theta - \hTh_{\tau_t}}_{M_t}^2 
&= \norm{  (\Theta -  \hTh_{\tau_t})^\top M_t (\Theta -  \hTh_{\tau_t}) }_2\\
&= \norm{  (\Theta -  \hTh_{\tau_t})^\top V_t^{\frac12} V_t^{-\frac12} M_t V_t^{-\frac12} V_t^{\frac12} (\Theta -  \hTh_{\tau_t}) }_2\\
&\le \norm{  (\Theta -  \hTh_{\tau_t})^\top V_t^{\frac12} }_2^2 \norm{V_t^{-\frac12} M_t V_t^{-\frac12}}_2
=\norm{  (\Theta -  \hTh_{\tau_t})^\top V_t^{\frac12} }_2^2 \norm{V_t^{-\frac12} }_{M_t}^2\,,
\end{align*}
where the last inequality follows because $\norm{\cdot}_2$ is an induced norm and induced norms are sub-multiplicative.
Hence, we have that
\begin{align*}
\sum_{t=1}^T \EE{\norm{\Theta-\widehat\Theta_{\tau_t}}_{M_t}} 
&\le \EE{\sum_{t=1}^T \norm{(\Theta-\widehat\Theta_{\tau_t})\ttop V_t^{1/2}}_2 
									\norm{V_t^{-1/2}}_{M_t} } \\
&\le \EE{ \sqrt{\sum_{t=1}^T \norm{(\Theta-\widehat\Theta_{\tau_t})\ttop V_t^{1/2}}_2^2 } \sqrt{\sum_{t=1}^T \norm{V_t^{-1/2}  }_{M_t}^2 }  } \\
&\le \sqrt{\EE{ \sum_{t=1}^T \norm{(\Theta-\widehat\Theta_{\tau_t})\ttop V_t^{1/2}}_2^2 }} \sqrt{ \EE{ \sum_{t=1}^T \norm{V_t^{-1/2}   }_{M_t}^2 } } \,,
\end{align*}
where the first inequality uses H\"older's inequality, and the last two inequalities use  Cauchy-Schwarz. By Lemma~\ref{lem:E3} in Appendix~\ref{app:lem}, using Assumption~\ref{ass:feature-mapping}, we have that
\[
\sum_{t=1}^T \min\left(1,  \smallnorm{V_t^{-1/2}}_{M_t}^2 \right) \le 2 m \log \left(\frac{\trace(V)+ T \Phi^2}{m}\right) \; .
\]
Denoting by $\lambda_{\min}(V)$ the minimum eigenvalue of $V$, a simple argument shows \todoc{Write it up, lemma in the appendix?}
$ \norm{V_t^{-1/2}}_{M_t}^2 \le  \norm{M_t}_2/\lambda_{\min}(V)\le \Phi^2/\lambda_{\min}(V) $, where in the second inequality we used Assumption~\ref{ass:feature-mapping} again.
Hence,
\begin{align*}
\sum_{t=1}^T \norm{V_t^{-1/2}}_{M_t}^2 
&\le \sum_{t=1}^T \min\left( \Phi^2/\lambda_{\min}(V),   \norm{V_t^{-1/2}  }_{M_t}^2\right) \\
&\le \sum_{t=1}^T \max\left( 1, \Phi^2/\lambda_{\min}(V) \right)  \min\left(1,  \norm{V_t^{-1/2}  }_{M_t}^2 \right) \; .
\end{align*}
Thus, 
\begin{align*}
\sum_{t=1}^T \EE{\norm{\Theta-\widehat\Theta_{\tau_t}}_{M_t}^2} 
&\le   \sqrt{ \EE{2 m \max\left( 1, \frac{\Phi^2}{\lambda_{\min}(V)} \right) \log \left( \frac{\trace(V)+T \Phi^2}{m} \right)}} \\
&\qquad\times\sqrt{\EE{\sum_{t=1}^T \norm{(\Theta-\widehat\Theta_{\tau_t})\ttop V_t^{1/2}}_2^2 } }  \; .
\end{align*}
By Lemma~\ref{lem:514} of  Appendix~\ref{app:lem} and the choice of $\tau_t$, we have that
\begin{align}
%\notag
\norm{(\Theta-\widehat\Theta_{\tau_t})\ttop V_t^{1/2}}_2 
&\le \sqrt{\frac{\det(V_t)}{\det(V_{\tau_t})}} \norm{(\Theta-\widehat\Theta_{\tau_t})\ttop V_{\tau_t}^{1/2}}_2 
\le \sqrt{2} \norm{(\Theta-\widehat\Theta_{\tau_t})\ttop V_{\tau_t}^{1/2}}_2 \; .
\label{eq:eq1}
\end{align}
Thus,
\begin{align*}
\EE{  \sum_{t=1}^T \norm{(\Theta-\widehat\Theta_{\tau_t})\ttop V_t^{1/2}}_2^2 }  
&\le 2 \EE{ \sum_{t=1}^T \norm{(\Theta-\widehat\Theta_{\tau_t})\ttop V_{\tau_t}^{1/2}}_2^2 }  
	&\text{(by \eqref{eq:eq1})} \\
&=  2\EE{  \sum_{t=1}^T \EE{   \norm{(\Theta-\widehat\Theta_{\tau_t})\ttop V_{\tau_t}^{1/2}}_2^2 \ \middle| \ \cF_{\tau_t} }  }   
	& \text{(by the tower rule)} \\
&\le  2 C T  \; . & \text{(by Assumption~\ref{ass:bnd-var})}
\end{align*}
Let $G_T =  2 m \max\left( 1, \frac{\Phi^2}{\lambda_{\min}(V) }\right) \log \left(\frac{\trace(V)+ T \Phi^2}{m} \right)$. Collecting the inequalities, we get
\begin{align*}
R_2 &= B \sum_{t=1}^T \EE{\norm{(\TTh_{\tau_t}-\Theta_*)\ttop\phi_t}} 
\le \sqrt{\EE{G_T}} \sqrt{C T} \\
&\le 4 B \sqrt{ m \max\left( 1, \frac{\Phi^2}{\lambda_{\min}(V)} \right) \log \left(\frac{\trace(V)+ T \Phi^2}{m} \right)}  \sqrt{C T} \; .
\end{align*}

\paragraph{Bounding $R_1$}
\label{sec:rare-switch}

If the algorithm has changed the policy $K$ times up to time $T$, then we should have that $\det(V_T)\geq 2^K$. On the other hand,
from Assumption~\ref{ass:feature-mapping}
 we have $\lambda_{\max}(V_T)\leq \trace(V)+(T-1)\Phi^2$. 
Thus, it holds that $2^K \leq (\trace(V)+ \Phi^2 T)^{m}$. Solving for $K$, we get $K\leq m\log_2 (\trace(V)+ \Phi^2 T)$. Thus,
\[
R_1 = H \sum_{t=1}^T \EE{\one{A_t}} \le H m\log_2 (\trace(V)+ \Phi^2 T) \; .
\]
Putting together the bounds obtained for $R_1$ and $R_2$, we get the desired result.
\end{proof}

%%%%%%%%%%%%%%%%%%%%%%%%%%%%%%%%%%%%%%%%%%%%%%%%%%%%%%%%

\begin{proof}[Proof of Theorem~\ref{thm:LQR++-stabilized}]
First notice that \cref{thm:LQR++} continues to hold if 
Assumption~\ref{ass:feature-mapping} is replaced by the following weaker assumption:
\begin{ass}[Boundedness Along Trajectories]
\label{ass:feature-mapping-2}
There exist $\Phi>0$ such that for all $t\ge 1$, $\EE{\trace(M(x_t,a_t))}\le \Phi^2$. 
\end{ass} 
The reason this is true is because \ref{ass:feature-mapping} is used only in a context where $\EE{\log( \trace(V+\sum_{s=1}^{T} M_t ) )}$ needs to be bounded. Using that $\log $ is concave, we get 
\[
\EE{\log( \trace(V+\textstyle\sum_{s=1}^{T} M_t ) )}
\le
\log\left( \EE{\trace(V+\textstyle\sum_{s=1}^{T} M_t ) }\right) \le \log( \trace(V)+T \Phi^2).
\]

With this observation,
the result follows from~\cref{thm:LQR++} applied to Lazy PSRL and $\{p'(\cdot|x,a,\Theta)\}$
as running Stabilized Lazy PSRL for $t$ time steps in $p(\cdot|x,a,\Theta_*)$ results in the same total expected cost
as running Lazy PSRL for $t$ time steps in $p'(\cdot|x,a,\Theta_*)$ thanks to the definition of Stabilized Lazy PSRL and $p'$.

Hence, all what remains is to show that the conditions of \cref{thm:LQR++} are satisfied when it is used with $\{p'(\cdot|x,a,\Theta)\}$. In fact, \ref{ass:ACOE} and \ref{ass:bnd-var} hold true by our assumptions. 
Let us check Assumption~\ref{ass:ACOE} next.
Defining $f'(x,a,\Theta,z) = f(x,a,\Theta,z)$ if $x\in \cR$ and $f'(x,a,\Theta,z) = f(x,\pistab(x),\Theta,z)$ otherwise, we see that $x_{t+1} = f'(x_t,a_t,\Theta,z_{t+1})$. Further, defining $M'(x,a) = M(x,a)$ if $x\in \cR$ and $M'(x,a) = M(x,\pistab(x))$ otherwise, we see that,
thanks to the second part that of \ref{ass:lindyn} applied to $p(\cdot|x,a,\Theta)$,
 for $y = f'(x,a,\Theta,z)$, $y' = f'(x,a,\Theta',z)$, 
$
\EE{\norm{y-y'}} \le \EE{ \norm{\Theta-\Theta'}_{M(x,a)} }
$
if $x\in \cR$ and 
$
\EE{\norm{y-y'}} \le \EE{ \norm{\Theta-\Theta'}_{M(x,\pistab(x))}} 
$
otherwise. Hence, $\EE{\norm{y-y'}} \le EE{ \norm{\Theta-\Theta'}_{M'(x,a)}}$, thus showing that \ref{ass:lindyn} holds for $p'(\cdot|x,a,\Theta)$ when $M$ is replaced by $M'$.
Now, Assumption~\ref{ass:feature-mapping-2} follows from Assumption~\ref{ass:stbl-cont}.

\end{proof}

%%%%%%%%%%%%%%%%%%%%%%%%%%%%%%%%%%%%%%%%%%%%%%%%%%%%%%%%

\begin{proof}[Proof of Corollary~\ref{cor:LQR++-finite}]
We prove the corollary by showing that conditions of Theorem~\ref{thm:LQR++} are satisfied.
 
\underline{Smoothly Parameterized Dynamics}: 
Because $\EE{y | x,a} = \Theta^\top \phi (x,a)$, $\EE{y' | x,a} = \Theta^{'\top} \phi (x,a)$, and $y$ and $y'$ have only one non-zero element,
\begin{align*}
\EE{\norm{y - y'}} &= \sqrt{2} \Prob{y \neq y'} = \sqrt{2} \left(1 - \Prob{y = y'} \right) \\
&= \sqrt{2} \left(1 - \Theta_{(x,a), :}^\top \Theta_{(x,a),:}' \right) = \frac{\sqrt{2}}{2} \norm{\Theta_{(x,a),:} - \Theta_{(x,a),:}'}^2 \,,
\end{align*} 
where the last step holds by the fact that each row of $\Theta$ and $\Theta'$ sum to one. 

\underline{Concentrating Posterior}: 
Let $N = (\Theta_* - \widehat \Theta_t)^\top$, $\alpha_{s,a,s'} = \alpha_{s'} + c_t(s,a,s')$ and $\overline \alpha_{s,a} = \sum_{s'} \alpha_{s,a,s'} = V_{t,(n(a-1)+s,n(a-1)+s )}$. We have that
\todoc[size=\tiny]{Note to myself: Check these.}
\begin{align*}
\EE{   \norm{N V_{t}^{1/2}}^2 \ \middle| \ \cF_{t} } &\le \EE{   \norm{N V_{t}^{1/2}}_F^2 \ \middle| \ \cF_{t} } \\
&= \EE{\sum_{s,a} V_{t,(n(a-1)+s,n(a-1)+s )} \sum_{s'} N_{s', n(a-1)+s}^2  \ \middle| \ \cF_{t}} \\
&= \sum_{s,a} \overline \alpha_{s,a} \sum_{s'} \EE{ N_{s', n(a-1)+s}^2  \ \middle| \ \cF_{t}} \; .
\end{align*}
Because each row of $\Theta_*$ has a Dirichlet distribution and rows of $\widehat \Theta_t$ are means of these distributions, $\EE{ N_{s', n(a-1)+s}^2  \ \middle| \ \cF_{t}}$ is simply the variance of the corresponding Dirichlet variable. Thus,   
\begin{align*}
\EE{   \norm{N V_{t}^{1/2}}^2 \ \middle| \ \cF_{t} } &\le \sum_{s,a} \sum_{s'} \frac{\overline \alpha_{s,a} \alpha_{s,a,s'} (\overline \alpha_{s,a} -\alpha_{s,a,s'})}{\overline \alpha_{s,a}^2 (1+\overline \alpha_{s,a})} \le n^2 d \; .
\end{align*}

%\underline{Existence of Regular ACOE Solutions}: 

\underline{Boundedness}: this holds by finiteness of the state space.  

\end{proof}

\begin{proof}[Proof of Corollary~\ref{cor:LQR++-linearly}]
We prove the corollary by showing that conditions of Theorem~\ref{thm:LQR++-stabilized} are satisfied.
 
\underline{Smoothly Parameterized Dynamics}: 
Because $y  = \Theta^\top \phi (x,a) + w$, $y'  = \Theta^{'\top} \phi (x,a) + w$, we have 
\[
\norm{y - y'}^2 = \norm{\Theta - \Theta'}_{\phi(x,a) \phi(x,a)^\top }^2 \;.
\]

\underline{Concentrating Posterior}: 
Let $\Lambda$ be a random variable with probability distribution function
\[
P(\lambda) \propto \exp\left( -\frac{1}{2} \left(\lambda - \widehat \Theta_{t,(:,i)} \right)^\top V_t \left( \lambda - \widehat \Theta_{t,(:,i)} \right) \right) \; .
\]
Notice that $\left( \Lambda - \widehat \Theta_{t,(:,i)} \right)^\top V_t^{1/2} = Z \sim \cN(0, I)$ has the standard normal distribution. Hence $\Prob{\abs{Z_j} > \alpha} \le e^{-\alpha^2/2}$. 
%Let $Y = (\Theta_i - \widehat\Theta_{t, i})\ttop V_{t}^{1/2}$. 
%We know that $\Prob{\Theta \in \cM} = \Prob{\Theta \in \cM \cap \cS}$. \todoc[inline]{Why is this last equality needed? Where is used?}
Thus, since $\Prob{\norm{Z} > \alpha} \le m e^{-\alpha^2/(2 m^2)}$, we have
\[
\EE{   \norm{ \left(\Theta_{*, (:,i)} - \widehat\Theta_{t, (:,i)} \right)\ttop V_{t}^{1/2}}^2 \ \middle| \ \cF_{t} } = \EE{\norm{Z}^2 \ \middle| \ \cF_{t}} = \int_{0}^\infty \Prob{\norm{Z}^2 > \epsilon} \le 2 m^3 \; .
\]
Thus, 
\begin{align*}
\EE{   \norm{(\Theta_*-\widehat\Theta_{t})\ttop V_{t}^{1/2}}^2 \ \middle| \ \cF_{t} } &\le \EE{   \norm{(\Theta_*-\widehat\Theta_{t})\ttop V_{t}^{1/2}}_F^2 \ \middle| \ \cF_{t} } \\
&= \sum_{i=1}^n \EE{   \norm{ \left(\Theta_{*,(:,i)} - \widehat\Theta_{t, (:,i)} \right)\ttop V_{t}^{1/2}}^2 \ \middle| \ \cF_{t} } \\
&\le 2 n m^3 \;.
\end{align*}
This shows that Assumption~\ref{ass:bnd-var} is satisfied. 

%\underline{Existence of Regular ACOE Solutions}: 

%\underline{Boundedness}: this holds by finiteness of the state space.  

\end{proof}